\documentclass[11pt]{article}

\usepackage{tabularx}
\usepackage{graphicx,subfigure}
\usepackage{fullpage}
\usepackage{hyperref}
\usepackage{natbib}
\usepackage{delarray}
\usepackage{times}
\usepackage{color}
\usepackage{amssymb}
\usepackage{amsmath}
\usepackage{amsthm}
\usepackage{mathrsfs}
\usepackage{tikz}
\usepackage{xspace}
\usepackage{enumerate}
\usepackage{algorithm}
\usepackage{algorithmic}
\usepackage{multirow}
\usepackage{nameref}
\usepackage{booktabs}
\usepackage{mdframed}
\usepackage{fmtcount}
\usepackage{footnote}
\usepackage{thmtools,thm-restate}
\usepackage{makecell}

\usepackage{framed}
\usepackage{float}

\newtheorem{theorem}{Theorem}%
\newtheorem{lemma}[theorem]{Lemma}

\newtheorem{definition}[theorem]{Definition}

\newcommand{\OO}{\ensuremath{\mathcal{O}}}
\newcommand{\E}{\ensuremath{\mathbf{E}}}
\newcommand{\junk}[1]{{}}
\newcommand{\cnt}{\ensuremath{\mathrm{cnt}}}
\newcommand{\Y}{\mathop{\mathbf{Yes}}}
\newcommand{\eps}{\ensuremath{\varepsilon}}
\newcommand{\rG}{\ensuremath{\mathcal{D}}}
\newcommand{\N}{\mathop{\mathbf{No}}}

\title{Towards a Query-Optimal and Time-Efficient Algorithm for Clustering with a Faulty Oracle}
\author{Pan Peng\thanks{Department of Computer Science, University of Sheffield. Email: {p.peng@sheffield.ac.uk}}\\ \and  Jiapeng Zhang\thanks{Department of Computer Science, University of Southern California. Email: {jiapengz@usc.edu}}}

\usepackage{times}
\date{}

\begin{document}
	
	\maketitle
	
	\begin{abstract}%
	Motivated by applications in crowdsourced entity resolution in database, signed edge prediction in social networks and  correlation clustering, Mazumdar and Saha [NIPS 2017] proposed an elegant theoretical model for studying clustering with a faulty oracle. In this model, given a set of $n$ items which belong to $k$ unknown groups (or clusters), our goal is to recover the clusters by asking pairwise queries to an oracle. This oracle can answer the query that ``do items $u$ and $v$ belong to the same cluster?''. However, the answer to each pairwise query errs with probability $\eps$, for some $\eps\in(0,\frac12)$. Mazumdar and Saha provided two algorithms under this model: one algorithm is query-optimal while time-inefficient (i.e., running in quasi-polynomial time), the other is time efficient (i.e., in polynomial time) while query-suboptimal. Larsen, Mitzenmacher and Tsourakakis [WWW 2020] then gave a new time-efficient algorithm for the special case of $2$ clusters, which is query-optimal if the bias $\delta:=1-2\eps$ of the model is large. It was left as an open question whether one can obtain a query-optimal, time-efficient
	algorithm for the general case of $k$ clusters and other regimes of $\delta$. 
	
	In this paper, we make progress on the above question and provide a time-efficient algorithm with nearly-optimal query complexity (up to a factor of $O(\log^2 n)$) for all constant $k$ and any $\delta$ in the regime when information-theoretic recovery is possible. Our algorithm is built on a connection to the stochastic block model. 
	\end{abstract}

\section{Introduction}

Clustering is a fundamental problem in machine learning with many applications. In this paper, we study an elegant theoretical model proposed by
\cite{mazumdar2017clustering} for studying clustering with the help of a faulty oracle. The model is defined as follows:

\paragraph{Model} Given a set $V=[n]:=\{1,\cdots,n\}$ of $n$ items which contains $k$ latent clusters $V_1,\cdots,V_k$ such that $\cup_i V_i=V$ and for any $1\leq i<j\leq k$, $V_i\cap V_j=\emptyset$. The clusters $V_1,\dots,V_k$ are unknown. We wish to recover them by making pairwise queries to an oracle $\OO$, which answers if the queried two vertices belong to the same cluster or not. This oracle gives correct answer with probability $1-\eps$, for some $\eps\in(0,\frac12)$. That is, for any  vertices $u,v\in V$, if $u$ and $v$ belong to the same cluster, then
\begin{equation*}
	\OO(u,v)=\begin{cases}
		+ &\textrm{with probability $1-\eps$, }\\
		- & \textrm{with probability $\eps$, }
	\end{cases}
\end{equation*}
and if $u,v$ belong to two different clusters, then
\begin{equation*}
	\OO(u,v)=\begin{cases}
		+ &\textrm{with probability $\eps$, }\\
		- & \textrm{with probability $1-\eps$.}
	\end{cases}
\end{equation*}
Equivalently, the model can be formalized as follows: Define a function $\tau: V\times V\to \{\pm 1\}$ such that   $\tau(u,v)=1$ if $u,v$ belong to the same cluster and $\tau(u,v)=-1$ if $u,v$ belong to different clusters. For any $u,v$, let $\eta_{u,v}\in \{\pm 1 \}$ be a random noise in the edge observation such that $\E[\eta_{u,v}]=\delta$. The noises $\eta_{u,v}$ are independent for all pairs $u,v\in V$. Then the oracle $\OO$ returns the sign of 
\[
\tau(u,v) \eta_{u,v}
\]
when the pair $u,v$ is queried. Note that by the above two formalization, it holds that $1-\eps=\frac{1}{2}+\frac{\delta}{2}$. In the following, we call $\delta$ the \emph{bias} of the model.  

It is assumed that repeating the same question to the oracle $\OO$, it always returns the same answer. (This was known as \emph{persistent noise} in the literature; see e.g. \citep{goldman1990exact}.) Our goal is to recover the latent clusters \emph{efficiently} (i.e., within polynomial time) with high probability by making as few queries to the oracle $\OO$ as possible.

\paragraph{Motivations} The above model captures several fundamental applications. In the \emph{entity resolution} (also known as the \emph{record linkage}) problem \citep{fellegi1969theory}, the goal is to find records in a data set that refer to the same entity across different data sources. Currently fully automated techniques for entity resolution has been unsatisfactory and current crowdsourcing platforms use human in the loop to help improve accuracy (see e.g. \citep{karger2011iterative,wang2012crowder,dalvi2013aggregating,gokhale2014corleone,vesdapunt2014crowdsourcing,mazumdar2017theoretical}). That is, the workers are asked to answer if any two items $u,v$ represent the same entity. It has been noted that the answers from non-expert workers are inevitably noisy. Furthermore, the goal of these crowdsourcing platforms is to use minimal number of queries to reduce cost and time for recovering the entities (clusters), which can be well modelled by the clustering with a faulty oracle.  

Another motivation is to predict the signed edges in a social network \citep{leskovec2010predicting}, where the sign (`$+$' or `$-$') on an edge indicates positive relation or negative relation between the corresponding two nodes. %
This problem can arise in many scenarios, e.g., voting on Wikipedia \citep{burke2008mopping} and making friends on Slashdot \citep{brzozowski2008friends}. Theoretically, there has been a line of work %
\citep{chen2014clustering,mitzenmacher2016predicting} 
that considers the model that allows the algorithm to query the sign of an edge $(u,v)$, which in turn can indicate whether $u,v$ belongs to the same cluster or not. It is further assumed that the answer to each query is correct with probability $1-\eps$, for some $\eps\in(0,\frac12)$. Thus, their model is also well captured by the previous model of clustering with a faulty oracle. {There is also some other related work on edge classification \citep{cesa2012correlation}.}

In addition, the model of clustering with noisy oracle is closely related to the problem of correlation clustering. In the correlation clustering problem \citep{bansal2004correlation}, we are given an undirected signed graph, and our goal is to partition the vertex set into clusters so that   the number of agreements\footnote{These are the number of $+$ edges inside clusters plus the number of $-$ edges between clusters.} is maximized or the number of disagreements\footnote{These are the number of $-$ edges inside clusters plus the number of $+$ edges between clusters.} is minimized. This problem is NP-hard and several approximation algorithms have been provided. In a variant formalization called noisy correlation clustering \citep{bansal2004correlation,mathieu2010correlation}, after given the ground truth clustering, the sign of each edge is flipped with some probability $\eps$. If the original graph is complete, then this is exactly the input of the problem of clustering with a faulty oracle. 

Finally, the model is strongly connected to the stochastic block model (SBM), which is popular model for studying graph clustering algorithms. In the SBM with parameters $N,k,p,q$ such that $0\leq q<p\leq 1$, denoted by SBM($N,k,p,q$), there is a set $V$ of $N$ vertices with a hidden $k$-partition $V_1,\cdots,V_k$ such that $\cup_i V_i=V$, where each part $V_i$ is called a \emph{cluster}. A graph $G$ is generated from the SBM($N,k,p,q$) model, if for any two vertices $u,v\in V$, an edge is added between $u,v$ with probability $p$ if $u,v$ are from the same cluster, and with probability $q$ if $u,v$ are from two different clusters. There has been a vast amount of research on recovering the underlying clusters from the SBM with different ranges of parameters in the past decade (see the recent survey \citep{abbe2017community}). Consider the noisy clustering model with and parameters $n,k,\delta$. Suppose that we make queries on all pairs $u,v\in V$, then the graph $G$ that is obtained by adding all $+$ edges answered by the oracle $\OO$ is exactly the graph that is  generated from the SBM model with parameters $N=n$, $k$,  $p=\frac{1}{2}+\frac{\delta}{2}$ and $q=\frac{1}{2}-\frac{\delta}{2}$. However, in our problem, our goal is to recover the clusters by making \emph{sublinear} number of queries, i.e., without seeing the whole graph. 

\paragraph{State-of-the-art}  %
\cite{mazumdar2017clustering} gave an \emph{inefficient} algorithm that perform $O(\frac{nk\log n}{\delta^2})$ queries to the oracle that recovers all the clusters of size $\Omega(\frac{\log n}{\delta^2})$. The query complexity of this algorithm nearly matches an information-theoretic lower bound $\Omega(\frac{nk}{\delta^2})$ presented by the same authors. The running time of their algorithm is $O\left(\left({k\log n}/{\delta^2}\right)^{(\log n)/\delta^2}\right)$, which is quasi-polynomial, and there is an inherent obstacle to push this algorithm to be efficient (see Section \ref{sec: optimal} for more details). Towards efficient algorithms, they designed another algorithm that runs in time $O(\frac{nk\log n}{\delta^2}+k(\frac{k^2\log n}{\delta^4})^\omega)$, makes $O(\frac{nk\log n}{\delta^2}+\min\{\frac{nk^2\log n}{\delta^4},   \frac{k^5\log^2 n}{\delta^8} \})$ queries and recovers all clusters of size at least $\Omega(\frac{k\log n}{\delta^4})$, where $\omega$ is the matrix multiplication exponent. %

In a follow-up work,  \cite{green2020clustering} proposed an improved algorithm for the case $k=2$, i.e., two clusters. This algorithm runs in time $O(\frac{n\log n}{\delta^2}+\frac{\log^3 n}{\delta^8})$ and makes $O(\frac{n\log n}{\delta^2}+\frac{\log^2 n}{\delta^6})$ queries. See Table \ref{tb:comparison} for a comparison of these results. 

Note that the above two efficient algorithms are query-suboptimal when $\delta$ is small, i.e., $\delta = o (n^{-1/4})$, even for $k=2$.  Due to this, \cite{green2020clustering} raised the following open question: 

\begin{quote}
	``\emph{Can we design a query-optimal, time-efficient algorithm that performs $O(\frac{kn\log n}{\delta^2})$ queries for all $0<\delta <1$?}''
\end{quote}
It is the main question we are trying to address in this paper. Note that for any non-trivial algorithm with query complexity $O(\frac{nk\log n}{\delta^2})$, it suffices to assume that $\delta\geq(k\log n/n)^{1/2}$, as the maximum number of queries one can make is $n^2$.
\subsection{Our results}

We give an algorithm with the following performance guarantee for the problem of clustering with a faulty oracle. 
\begin{theorem}\label{thm:main}
	There exists a polynomial time algorithm $\textsc{NosiyClustering}$ %
	that recovers all the clusters of size $\Omega(\frac{k^4\log n}{\delta^2})$ with success probability $1-o_n(1)$. The total number queries that $\textsc{NosiyClustering}$ performs to the faulty oracle $\OO$ is $O(\frac{n k\log n}{\delta^2}+\frac{k^{10}\log^2 n}{\delta^4})$.
\end{theorem}

\begin{table*}[t]
	\centering
	\begin{tabular}{|c|c|c|c|}
		\hline
		\textbf{\# clusters}	& \textbf{query complexity} & \textbf{time-efficient ?} & \textbf{reference}   \\
		\hline
		\multirow{3}{0em}{$k$}	& $O(\frac{nk\log n}{\delta^2})$ & $\N$ &   \multirow{3}{10em}{\citep{mazumdar2017clustering}} \\ \cline{2-3}
		& %
		$O(\frac{nk\log n}{\delta^2}+\min\{\frac{nk^2\log n}{\delta^4}, \frac{k^5\log^2 n}{\delta^8} \})$
		& %
		$\Y$
		&    \\ \cline{2-3}
		&  $\Omega(\frac{nk}{\delta^2})$ & \textbf{Lower bound} &    \\
		\hline
		$2$ &$O(\frac{n\log n}{\delta^2}+\frac{\log^2 n}{\delta^6})$ & $\Y$ & \citep{green2020clustering}\\
		\hline 
		\hline
		\multirow{2}{0em}{$k$}
		& nearly-balanced: $O(\frac{nk\log n}{\delta^2}+\frac{k^4\log^2n}{\delta^4})$& $\Y$ & \multirow{2}{10em}{\textbf{this work}}\\ \cline{2-3}
		
		&$O(\frac{n k\log n}{\delta^2}+\frac{k^{10}\log^2 n}{\delta^4})$& $\Y$ & \\ %

		\hline 
		
	\end{tabular}
	\caption{Comparison of algorithms for clustering with a faulty oracle. We say an algorithm is time-efficient, if it runs in polynomial time (in $n,k,1/\delta$). We stress that all the upper bound holds for algorithms success probability at least $1-o_n(1)$, while the lower bound is for any algorithm with constant success probability.}\label{tb:comparison}
\vspace{1em}
\end{table*}

Note that for any constant $k$, the {query complexity of our algorithm} \textsc{NoisyClustering} in Theorem \ref{thm:main} is 
\begin{equation*}
	O(\frac{n \log n}{\delta^2}+\frac{\log^2 n}{\delta^4}) =
	\left\{\begin{array}{cc}
		O(\frac{n\log n}{\delta^2})    & \text{if $\delta=\omega((\frac{\log n}{n})^{1/2})$} \\
		O(\frac{n\log^2 n}{\delta^2})     & \text{if $\delta\in [\Omega((\frac{1}{n})^{1/2}), O((\frac{\log n}{n})^{1/2}))$} 
	\end{array}\right.
\end{equation*}
Thus, as long as $\delta=\Omega((\frac{1}{n})^{1/2})$ (i.e., $\delta$ is in the regime when information-theoretic recovery is possible), our algorithm achieves nearly-optimal query complexity (up to a factor of $O(\log^2 n)$). On the other hand, if $\delta=o((\frac{1}{n})^{1/2})$, it is impossible to recover the latent clusters, which follows from the information-theoretic lower bound $\Omega(\frac{n}{\delta^2})$ and an inherent restriction on the maximum number of queries, i.e., $n^2$, as there are at most $n^2$ edges.  Therefore, %
\emph{we {almost} fully resolve the aforementioned open question by \cite{green2020clustering} for any constant $k\geq 2$. }

The main focus on this paper is to optimize the dependency on $\delta$. We do not attempt to optimize the dependency on $k$. By combining ideas from \cite{mazumdar2017clustering}, we believe it is possible to slightly improve the term $k^{10}$. However several evidences suggested there is an inherent obstacle to match the information theoretical lower bound by efficient algorithms. See Section \ref{sec: optimal} for more details. The algorithm \textsc{NoisyClustering} is built upon a simple algorithm  for the case that the underlying clustering $V_1,\dots,V_k$ are nearly-balanced, i.e., each cluster $V_i$ has size $\Omega(\frac{n}{k})$. For the latter case, we achieve a slightly better algorithm. Formally, we define a $b$-balanced partition as follows. 
\begin{definition}
	Let $b\in [0,1]$. Given a vertex set $V$ and a partition $V_1,\dots, V_k$ such that $\cup_i V_i=V$, we call $V_1,\dots,V_k$ a $b$-balanced partition, if for each $i$, $|V_i|\geq b n/k$. 
\end{definition}
We show the following result for %
the case that the underlying partition is the $b$-balanced. 
\begin{theorem}
	\label{thm: balanced}
	Let $b\in(0,1]$. Let $n\geq \frac{C_0k^2\log^2 n}{b^2\delta^2}$ for some constant $C_0>0$. Suppose that the underlying partition $V_1,\dots,V_k$ of $V=[n]$ is $b$-balanced. There is a polynomial time algorithm %
	that recovers all the clusters with success probability $1- o_n(1)$. 
	The total number queries that the algorithm performs to the faulty oracle $\mathcal{O}$ is $O(nk\cdot\log n/\delta^2 +k^4\cdot\log^2n/(b^4\delta^4) )$.
\end{theorem}

For any constant $b>0$, the query complexity of the above algorithm is $O(\frac{k^{4}\log^2n}{\delta^4}+\frac{k n\log n}{\delta^2})$, %
which is in comparison to the information-theoretic lower bound $\Omega(\frac{nk}{\delta^2})$ that also holds for the nearly-balanced instance \citep{mazumdar2017clustering}. The query complexity almost matches the lower bound when $k = o((\delta^2\cdot n)^{1/3}$), which leaves open in the range $(\delta^2\cdot n)^{1/3}\leq k\leq \delta^2\cdot n$. Interestingly, there exists evidence suggesting that there is no \emph{efficient} algorithm matching the information theoretical lower bound when $k$ is large. We refer to Section \ref{sec: optimal} for a more detailed discussion.

\subsection{Discussion of previous approaches and an overview of our algorithms} %
We first sketch the main idea underlying the algorithms in \citep{mazumdar2017clustering,green2020clustering}. Their algorithms do the following: 
\begin{enumerate}
	\item select a subset $T$ of $t=\textrm{poly}(k\log n/\delta)$ vertices, and build a graph $H_T=(T,E_T)$ by making queries for all pairs $u,v\in T$ and defining the edge set $E_T$ according to the query answers;
	\item find all sub-clusters $X$ of size $\Omega(\frac{\log n}{\delta^2})$ from the $T$ by making use of the graph $H_T$, where a set $X$ is a sub-cluster if  $X\subseteq V_i$ for some cluster $V_i$;
	\item grow each of the sub-clusters $X$ to $V_i$: arbitrarily select a subset $X_0\subseteq X$ of size $\Theta(\frac{\log n}{\delta^2})$ and add all vertices $v\in V$ to $X$ such that the number of `$+$' neighbors of $v$ in $X_0$ is more than  $\frac{|X_0|}{2}$.
\end{enumerate}
Then the algorithm removes all the identified clusters and repeat the above process if the number of remaining vertices is still large and more clusters need to be identified.

Both of the previous two efficient algorithms are based on some `local' approaches of finding sub-clusters from $H_T$ (in Step 2 above), i.e., by counting the number of `$+$' neighbors and/or shared neighbors of vertices in $T$. Such `local' approaches require the algorithm to choose a large subset $T$ whose size eventually results in the sub-optimality of the total number of queries to the oracle. We also note that the query-optimal algorithm in  \citep{mazumdar2017clustering} is a `global' approach in the sense that it makes use of a large subgraph of $H_T$ to cluster the vertices in $T$. However their subroutine for finding the subgraph requires quasi-polynomial time, which can not be improved to polynomial time, assuming that the hidden clique problem is hard in average case, which is a well-believed assumption in complexity theory.

\paragraph{Our approach.} Our algorithm is built upon the same framework, while uses several new ideas. One of our key observations is that we can make use of the `global' and time-efficient algorithms for clustering graphs generated from SBM with appropriate parameters to find sub-clusters in the small representative graph $H_T$, when the input instance is nearly-balanced. Slightly more precisely, note that for any subset $T\subset V$, if we let $E_T$ be the set of all `$+$' edges from the query answers and let $H_T=(V,E_T)$, then we can equivalently view $H_T$ as generated from the stochastic block model SBM($|T|,k,p,q$) with $p=\frac{1}{2}+\frac{\delta}{2}$, $q=\frac{1}{2}-\frac{\delta}{2}$. Previous research (e.g,. \citep{mcsherry2001spectral,vu2018simple}) suggests that if $H_T$ contains $k$ nearly-balanced clusters and the parameters $|T|,p,q,k$ satisfy certain conditions (see Theorem \ref{thm:Vu_original}), then with high probability, we can efficiently recover all the clusters in $T$. Now if the original instance $V_1,\dots,V_k$ is nearly-balanced (i.e., $|V_i|\geq \frac{bn}{k}$, $i\leq k$, for some constant $0<b<1$), then we can show that a randomly sample set $T$ with $\Theta(\frac{k^2\log n}{\delta^2})$ vertices will satisfy both the nearly-balanced requirement of $H_T$ and the condition for clustering SBM. Then by applying one algorithm (specifically, Vu's algorithm; see Theorem \ref{thm:Vu18}) for clustering the graph $H_T$ from SBM($|T|,k,p,q$) to find all the sub-clusters $X_1,\dots,X_k$, and growing each sub-cluster as described before, we obtain our algorithm for clustering the nearly-balanced instance with improved performance guarantee. We give details in Section \ref{section: balanced}.

For the unbalanced instance, i.e., there exists at least one cluster of size less than $\frac{b n}{k}$, we have to modify this algorithm since  unbalanced instance is a barrier to algorithms for the stochastic block model. Our second observation is that there must exist a \emph{size-gap} between different clusters, which allows us to filter out the small size clusters. The remaining large clusters are again nearly-balanced (with different balance ratio), which can be clustered as before. Concretely, let $s_1\geq \dots\geq s_k$ be the size of each cluster. If $s_{k}<\frac{bn}{k}$, we show there is a $\mu>0$ and $h\in[k]$ such that,
\[
s_1\geq\cdots \geq s_{h} \geq \mu\cdot n >(\mu - b\cdot k^{-2})\cdot n \geq s_{h+1}\cdots \geq s_{k}.
\]
Notice that for every $i\le h$ and $v\in V_{i}$, the expectation of the degree of $v$ in the random graph $G$ is
\[
\underset{G}{\E}\left[|\{u: (u,v)\in E(G)\}|\right] = \left(\frac12+\frac{\delta}{2}\right) |V_{i}| + \left(\frac12-\frac{\delta}{2}\right) (n-|V_{i}|) \geq  \left(\frac12-\frac{\delta}{2}\right) n + \delta\mu n
\]
On the other hand, for each $i'>h$ and $v'\in V_{i'}$, the expectation of degree of $v$ is 
\[
\underset{G}{\E}\left[|\{u: (u,v')\in E(G)\}|\right] = \left(\frac12+\frac{\delta}{2}\right) |V_{i}| + \left(\frac12-\frac{\delta}{2}\right) (n-|V_{i}|)\leq  \left(\frac12-\frac{\delta}{2}\right) n + \delta\mu n -\delta \cdot b\cdot k^{2} n
\]
Therefore, there is a $\delta \cdot b\cdot k^{2} n$ gap between large clusters and small clusters (in expectation). It is easy to show that the gap also exists with high probability by applying the standard concentration bound.

Now if we sample a subset $T$ of size at least $\Omega(\frac{k^4\log n}{\delta^2})$, then we can guarantee that with high probability, for all vertices in large clusters $V_i$ ($i\leq h$), they have degree larger than some threshold $d_h$ in $H_T$, while for all vertices in small clusters   $V_i$ ($i>h$), they have degree smaller than $d_h$ in $H_T$. In this way, we can filter out all vertices in $T$ that belong to small clusters and let the remaining vertex set be $T'$ and the corresponding subgraph be $H_{T'}$. Then we can run Vu's algorithm on $H_{T'}$ to identify all the sub-clusters in $T'$ that corresponding to large clusters in $G$. However, there is one subtle issue in the above approach, that is, we do not know the index $h$ that corresponds to the size-gap. To resolve this issue, we simply try all possible candidates $h$: for each $h\in [k]$, we pretend that $h$ is the index corresponding to the size-gap of the clusters. Then we use $h$ to obtain a filtered subgraph $H_{T'}$ and invoke Vu's algorithm on $H_{T'}$ to find $h$ sets $X_1,\dots,X_h$. Now we give a simple algorithm to test if $h$ is the `right' index, by testing if all sets $X_i$ are \emph{biased} towards some true cluster $C$ or not, i.e., if the the majority of $X_i$ belong to $C$. We can show that if for an index $h$, all the sets $X_1,\dots,X_h$ pass the bias testing, then we can still use each $X_i$ to grow the cluster. Finally, if $h$ is the index that corresponds to size-gap, then it will pass the test with high probability by the previous argument, which ensures that we can always find some clusters in this way.     
We give details in Section \ref{sec:unblanced}.

\subsection{Towards optimal dependency on the number of clusters}
\label{sec: optimal}
As mentioned before, our algorithm (in Theorem \ref{thm: balanced}) for clustering nearly-balanced instances makes $O(k\cdot n\log n/\delta^2+k^4\log^2n/\delta^4)$ queries, which is  in comparison to the known lower bound $\Omega(k\cdot n/\delta^2)$  \citep{mazumdar2017clustering}. %
There exists evidence indicating that our query complexity might be almost optimal, in particular, improving the factor $k^4$ in the second term of the query complexity seems difficult when $k$ is large. %

Several papers \citep{decelle2011asymptotic, chen2014improved} suggested that, using non-rigorous but deep arguments from statistical physics, efficiently recovering the clusters in SBM($N,p,q,\delta$) is impossible if $\frac{p-q}{\sqrt{p}}=o(\frac{\sqrt{N}}{s})$, where $s$ is the size of minimum cluster. Translating it to our case with $N=n$,  $p=\frac{1}{2}+\frac{\delta}{2}$, $q=\frac{1}{2}-\frac{\delta}{2}$ and $s=\Omega(\frac{n}{k})$, it suggests that even if we query the whole graph (i.e., with $\Theta(n^2)$ queries), it is impossible to recover the clusters if $k = \omega(\delta\sqrt{n}).$ On the other hand, suppose that there exists a polynomial time algorithm $\mathcal{B}$ that solves our problem with query complexity $O(kn/\delta^2+k^{4-\varepsilon}/\delta^4)$ for any constant $\varepsilon>0$, then it can recover the clusters in the corresponding SBM model by querying $o(n^2)$ pairs, for $k=\delta n^{\frac{1}{2}+\frac{\varepsilon}{10}}=\omega(\delta\sqrt{n})$, which seems impossible by the aforementioned evidence. 

It will be very interesting to formally prove that the query complexity $O(k\cdot n\log n/\delta^2+k^4\log^2n/\delta^4)$ of the algorithm in Theorem \ref{thm: balanced} is almost optimal (up to a $\log^2 n$ factor) for any \emph{polynomial time} algorithm, by assuming some standard hardness assumptions (e.g. finding a random clique is hard) in complexity theory. In fact, Mazumdar and Saha \citep{mazumdar2017clustering} also pointed it is impossible to push their query-optimal algorithm to be efficient unless there is an efficient algorithm  finding hidden clique in random graphs .%

\section{Two Subroutines}\label{sec:subrountines}
We now introduce two subroutines, which will be used in our clustering algorithms later. 
\subsection{An algorithm for nearly balanced clustering in stochastic block model}
For convenience of notation, we introduce the following. Fix any $k$ clusters $V_1,\dots,V_k$ and a bias parameter $\delta\in[0,1)$.  The distribution $\rG(V_1,\dots,V_{k}, \delta)$ samples a random graph as follows: for any two vertices $u$ and $v$, we add an edge between them with probability $(1/2+\delta/2)$ if $u$ and $v$ come from the same cluster $V_{i}$, and add an edge between them with probability $(1/2-\delta/2)$ otherwise. The goal of the clustering algorithm is to recover the clusters $V_1,\dots,V_k$ though a random graph $G\sim\rG(V_1,\dots,V_k,\delta)$. %

We first note that the following result was implicitly shown in \cite{vu2018simple}. 
\begin{theorem}[\cite{vu2018simple}]%
	\label{thm:Vu18}
	Let $\delta\in[0,\frac{1}{2}]$ and $G\sim\rG(V_1,\dots,V_k,\delta)$. Let $n = |V_1|+\dots+|V_{k}|$.  Suppose that the partition $V_1,\dots,V_k$ is $b$-balanced for some $b\in (0,1]$. Then there exists an algorithm, denoted by \textsc{BalPartition}($G,k,\delta,b$), that recovers all the clusters $V_1,\dots,V_k$ of $G$ in polynomial time with probability at least $1-n^{-8}$, if the following condition holds,
	\[
	n\geq c_0 \frac{k^2}{b^2\delta^2}\log n, 
	\]
	where $c_0> 1000$ is some universal constant.  
\end{theorem}
This theorem is slightly different from the original version of \cite{vu2018simple}, and we present an explanation in Appendix \ref{app:Vu18}.

\subsection{Growing a cluster from a biased set}
All our algorithms will make use of a subroutine (Algorithm \ref{alg:belongto}) for classifying vertices in $V$ with the help of a \emph{biased} set $B$, of which the majority belong to the same cluster. More formally, we give the following definition.
\begin{definition}
	Let $\eta\in [0,\frac12]$.	Let $C$ be a true cluster, i.e., $C=V_i$ for some  $i\in [k]$. A set of vertices $B$ is called  $(\eta,C)$-biased if $|B\cap C|\geq (1/2+\eta)\cdot |B|$.
\end{definition}
Note that if $\eta=\frac{1}{2}$, then all the vertices in set $B$ are contained in $C$, i.e., $B\subseteq C$. In this case, we all $B$ a \emph{sub-cluster} of $C$. We now describe this subroutine and state its performance guarantee.

\begin{algorithm}[H]
	\caption{\textsc{BelongToCluster}($v,B$): test if $v$ belongs to a cluster $C$,  given a $(\eta,C)$-biased set $B$ }\label{alg:belongto}
	\begin{algorithmic}[1]
		\STATE Query all pairs $v,w$ for $w\in B$ and let $\cnt$ be the number of $+$ answers 
		\IF{$\cnt \geq \frac{|B|}{2}$}
		\RETURN $\Y$
		\ELSE
		\RETURN $\N$
		\ENDIF
		
	\end{algorithmic}
	
\end{algorithm}

\begin{lemma}\label{lemma:setB}
	Let $B$ be a set that is $(\eta,C)$-biased and have size at least $\frac{16 \log n}{\eta^2\delta^2}$. Then with probability at least $1-n^{-7}$, 
	\begin{itemize}
		\item for all vertices $v\in C$, \textsc{BelongToCluster}($v,B$) returns $\Y$;
		\item for all vertices $v\in V\setminus C$, \textsc{BelongToCluster}($v,B$) returns $\N$.
	\end{itemize}
\end{lemma}
Note that the above lemma says that by invoking \textsc{BelongToCluster}($v,B$) for any $v\in V$, we can identify all the cluster members in $C$ with high probability. %
\begin{proof}[Proof of Lemma \ref{lemma:setB}] 
	Let $v$ be an arbitrary vertex. Let $B_v$ denote the subset of vertices of $B$ that belong to the same cluster as $v$. Query all the edges between $v$ and $B$. Then the expected number of `$+$' neighbors of $v$ is 
	\[
	\left(\frac{1}{2}+\frac{\delta}{2}\right) |B_v| + \left(\frac{1}{2}-\frac{\delta}{2}\right) |B\setminus B_v|=\left(\frac{1}{2}-\frac{\delta}{2}\right) |B| + \delta |B_v| 
	\]
	Let $\lambda=\frac{\eta\delta |B|}{2}$. Note that $\lambda^2/|B|\geq 4\log n$ as $|B|\geq \frac{16\log n}{\eta^2\delta^2}$. Recall that $B$ is $(\eta, C)$-biased for some constant $\eta$ and cluster $C$. We consider two cases.  
	\begin{itemize}
		\item If $v\in C$, then $|B_v|\geq (\frac{1}{2}+\eta)|B|$ and the expected number of `$+$' neighbors of $v$ is at least 
		\[
		\left(\frac{1}{2}-\frac{\delta}{2}\right)|B|+\left(\frac{1}{2}+\eta\right)\delta |B| = \left(\frac{1}{2}+\eta \delta\right)|B|
		\] 
		
		By Chernoff--Hoeffding  bound (see Theorem \ref{thm:chernoff}), with probability at least $1-e^{-2\lambda^2/|B|}\geq  1-n^{-8}$, the number of `$+$' neighbors of $v$ is at least
		\begin{eqnarray}
			\left(\frac{1}{2}+\eta \delta\right)|B| -\lambda =\left(\frac{1}{2}+\frac{1}{2}\eta\delta\right)|B|>\frac{1}{2}|B|\label{eqn:biaslowerbound}
		\end{eqnarray}
		
		\item if $v\in C'$ for some cluster $C'\neq C$, then $|B_v|\le \left(\frac{1}{2}-\eta\right)|B|$, the expected number of `$+$' neighbors of $v$ is at most 
		\[
		\left(\frac{1}{2}-\delta\right)|B|+\left(\frac12-\eta\right)\delta |B| =\left(\frac{1}{2}-\eta \delta\right)|B|
		\] 
		By Chernoff--Hoeffding bound, with probability at least $1-e^{-2\lambda^2/|B|}\geq 1-n^{-8}$, the number of `$+$' neighbors of $v$ is at most
		\[
 \left(\frac{1}{2}-\eta \delta\right)|B|+\lambda  =\left(\frac{1}{2}-\frac{1}{2}\delta\eta\right)|B|<\frac{1}{2}|B|
		\]
	\end{itemize}
	
	Therefore, with probability at least $1-n^{-7}$, for each vertex $v\in V$, it holds that 
	\begin{itemize}
		\item if $v\in C$, then the number of $+$ neighbors is at least $\frac{1}{2}|B|$, and \textsc{BelongToCluster}($v,B$) returns $\Y$; and
		\item if $v\notin C$, then the number of $+$ neighbors is less than $\frac{1}{2}|B|$, and \textsc{BelongToCluster}($v,B$) returns $\N$.
	\end{itemize} 
\end{proof}

\section{Clustering Nearly-Balanced Instances}
\label{section: balanced}

In this section, we give our algorithm for clustering $b$-balanced instances, for any $b\in(0,1]$. It simply first invokes the following Algorithm \ref{alg:bal_cluster} and then Algorithm \ref{alg:subcluster_cluster}. It is built on the two subroutines \textsc{BalPartition} and \textsc{BelongToCluster} introduced in Section \ref{sec:subrountines}.

\begin{algorithm}[H]
	\caption{\textsc{BalancedClustering}$(V,k,\delta,b)$: clustering for a $b$-balanced instance}
	\label{alg:bal_cluster}
	\begin{algorithmic}[1]
		\STATE Let $n=|V|$, $b' =b/2$ and $c_0$ be the constant from Theorem \ref{thm:Vu18}
		\STATE Randomly sample a subset $T\subset V$ of size $|T| = \frac{400 c_{0} k^2 \log n }{b^2\delta^2}$ 
		\STATE Query all pairs $u,v\in T$ and let $H_T$ be graph on vertex set $T$ with only positive edges from the query answers
		\STATE Apply \textsc{BalPartition}$(H_T,k,\delta,b')$ to obtain clusters $X_1,\dots, X_k$
	\end{algorithmic}
\end{algorithm}

\begin{algorithm}[H]
	\caption{\textsc{GlobalGrow}$(V,X_1,\dots,X_k)$: from sub-clusters to clusters}
	\label{alg:subcluster_cluster}
	\begin{algorithmic}[1]
		\STATE Let $U=V$ and $n=|V|$ 
		\STATE For each $1\le i\leq k$, find an arbitrary subset $X_i'\subseteq X_i$ of size $\frac{1600\log n}{\delta^2}$
		\FOR{each $i\in[k]$} \STATE let 
		$C_{i}:=\{v\in U: \textsc{BelongToCluster}(v,X_i') \text{ returns $\Y$} \}$
		\STATE update $U\gets U\setminus C_i$
		\ENDFOR
		\RETURN $C_1,\cdots, C_{k}$
	\end{algorithmic}
\end{algorithm}

Now we provide the analysis of this algorithm, i.e., prove Theorem \ref{thm: balanced}. In the following, we let $T$ denote the sample set from \textsc{BalancedClustering}($V,k,\delta,b$). For each $i\in[k]$, let $T_i = T\cap V_i$ be the sub-clusters. We first show that, with high probability, the clusters $T_1,\dots,T_k$ are balanced.

\begin{lemma}
	\label{lemma: balanced sampling}
	Let $V_1,\dots,V_k$ be a family of $b$-balanced clusters. Then with probability at least $1-n^{-7}$, $T_1,\dots,T_k$ is $b'$-balanced. 
\end{lemma}
\begin{proof}
	Since $V_1,\dots,V_k$ is a family of $b$-balanced clusters, we have that $\E[|V_i \cap T|] \geq b\cdot |T|/k$. Notice that $T$ is a uniform random subset. By the Chernoff bound, for each $i$, with probability at least $1-n^{-8}$, $|T_i| \geq b'\cdot|T|/k$. The claim then follows by the union bound. 
\end{proof}

Now we may assume that $(T_1,\dots,T_k)$ is $b'$-balanced. Since the size of $T$ is large, i.e., $|T|=\frac{400 c_{0} k^2 \log n }{b^2\delta^2}=\frac{ 100\cdot c_{0}\cdot k^2 \log n}{b'^2\delta^2}$, we are able to recover the clusters in $T$ by Theorem \ref{thm:Vu18}.

\begin{lemma}
	\label{lemma: McSherry}
	Suppose that the partition $T_1,\dots,T_k$ of the sampled set $T$ is $b'$-balanced. Let $X_1,\dots,X_k$ be the output sets of \textsc{BalancedClustering}($V,k,\delta,b$). Then 
	\[
	\Pr[X_1,\dots,X_k\text{ is not a correct clustering of }H_T] \leq |T|^{-8}
	\]
\end{lemma}
\begin{proof}
	Note that by our choice of $|T|$ and that $b'=\frac{b}{2}$, we have
	$|T|\geq c_0 \frac{k^2}{b'^2 \delta^2}\log n.$ 
	Then the correctness of Lemma \ref{lemma: McSherry} simply follows by Theorem \ref{thm:Vu18}.
\end{proof}

Now we are ready to prove Theorem \ref{thm: balanced}. 
\begin{proof}[Proof of Theorem \ref{thm: balanced}] 
	By Lemma \ref{lemma: McSherry}, the output $X_1,\dots,X_k$ is a correct clustering of $H_T$, with probability $1-o_n(1)$. Conditioned on this, we know that each $X_i$ is $(\frac{1}{2},C)$-biased for some cluster $C$. This also implies that each $X_i'\subseteq X_i$ is $(\frac12,C)$-biased. Thus, by invoking \textsc{BelongToCluster}($v,X_i'$) for all $v\in V$ and $i\leq k$ and by Lemma \ref{lemma:setB} with $\eta=0.1<\frac{1}{2}$, we can guarantee that the output $C_1,\dots,C_k$ of \textsc{GlobalGrow}$(V,X_1,\dots,X_k)$ is a correct clustering with   probability $1-\Theta(|T|^{-8})=1-o_n(1)$. 
	
	Note that we query all the pairs $u,v\in T$, which corresponds to $|T|^2$ queries. Note further that there are at most $k$ clusters, each of which grows from a sub-cluster of size $\Theta(\frac{\log n}{\delta^2})$. In total, the query complexity of Algorithm \ref{alg:bal_cluster} and \ref{alg:subcluster_cluster} is upper bounded by 
	$O(|T|^2 + k\frac{\log n}{\delta^2}\cdot n) = O(k^4\cdot\log^2n/(b^4\delta^4) + nk\cdot\log n/\delta^2).$ 
	Since the running time of \textsc{BalPartition} is polynomial in $|T|,k,\delta,b$ and the running time for growing each of the clusters is linear in $n$, the total running time of our algorithm is polynomial (in $n,k,\delta,b$). 
\end{proof}

\section{Clustering the General Instances }\label{sec:unblanced}

In the section, we give our algorithm for the general instances. 

\subsection{Existence of size-gap in unbalanced instances}
We first focus on the unbalanced case, that is, the underlying clustering is not $b$-balanced, i.e., the size of the minimum cluster is less than $\frac{bn}{k}$. Let $V_1,\dots,V_k$ be a family of clusters, and let $s_1,\dots,s_{k}$ be the size of each cluster respectively. Without loss of generality, we assume that $s_1\geq \dots\geq s_k$. A useful observation is the following size-gap lemma. Roughly speaking, for any unbalanced clusters, there a threshold which separates large and small clusters. 

\begin{lemma}[size-gap]
	\label{lemma: sizegap}
	Let $b\in [0, \frac12]$. 
	If $s_k< \frac{b n}{k}$, then there exists $h<k$ such that 
	\begin{itemize}
		\item $s_h\geq \frac{n}{k} - \frac{h\cdot b\cdot n}{k^2}$, and
		$s_{h+1}< \frac{n}{k} - \frac{(h+1)\cdot b\cdot n}{k^2}$.
	\end{itemize}
	Hence the gap between $s_{h}$ and $s_{h+1}$ is at least $\frac{bn}{k^2}$. 
\end{lemma}
\begin{proof}
	Note that by averaging argument, it holds that $s_1\geq \frac{n-s_k}{k-1}\geq \frac{(1-b/k)n}{k-1}>\frac{(1-b/k)n}{k}=\frac{n}{k}-\frac{bn}{k^2}$. This implies that the subset $I\subseteq [k]$ of indices $i$ with 
	$s_{i} \geq \frac{n}{k} - \frac{i \cdot bn}{k^2} $ %
	is not empty. Let $h$ be the largest $i$ in the set $I$. Furthermore, since $s_k<\frac{bn}{k} \leq \frac{(1-b)n}{k}=\frac{n}{k} - \frac{k \cdot bn}{k^2}$ for any $b\leq \frac{1}{2}$, it must hold that $k\notin I$ and thus $h\leq k-1$. The statement of the lemma then follows from the choice of $h$. 
\end{proof}

\subsection{Recovering sub-clusters from the sampled subgraph with known gap}

From Lemma \ref{lemma: sizegap}, we know that in the unbalanced case, there is a size-gap between two clusters $V_h$ and $V_{h+1}$, for some index $h\leq k-1$. In the following, we first present an algorithm under the assumption that the index $h$ is known. Later, we show how to use this algorithm to deal with the general case.

\begin{algorithm}[H]
	
	\caption{\textsc{GapClustering}($V,h,\delta,b$): clustering with known size-gap}
	\label{algorithm: known_gap}
	\begin{algorithmic}[1]
		\STATE Let $n=|V|$ and sample a set $T \subset U$ of size $t=\frac{8c_0 k^4\log n}{b^2\cdot\delta^2}$%
		\STATE Query all pairs $u,v\in T$
		
		\STATE Let $H_T=(T,E_T)$ be graph on vertex set $T$ with only positive edges from the query answers
		
		\STATE\label{alg:remove} Remove all vertices in $H_T$ with degree less than $d_h:=\frac{t}{2}-\left(\frac12-\frac{1}{k}+\frac{(h+1/2)b}{k^2}\right)\delta t$
		
		\STATE Let $T'$ be the set of remaining vertices and let the resulting graph be $H_{T'}$
		
		\STATE Apply \textsc{BalPartition}($H_{T'},k,\delta,b'':=\frac{h}{2k}$) to find clusters $X_1,\dots, X_h$
	\end{algorithmic}
	
\end{algorithm}
The crucial idea of the above algorithm is that we are able to show the Step \ref{alg:remove} of Algorithm \ref{algorithm: known_gap} removes all vertices sampled from small clusters in $T$. Hence the remaining graph $T'$ becomes a nearly-balanced clustering instance, in which the sub-clusters correspond to large clusters $V_1,\dots,V_h$. We have the  following lemma regarding this algorithm. %
\begin{lemma}\label{lemma:knowngap}
	Let $b\in [0,\frac12]$. Suppose that %
	$s_h\geq \frac{n}{k} - \frac{h\cdot b\cdot n}{k^2}$, and $s_{h+1}< \frac{n}{k} - \frac{(h+1)\cdot b\cdot n}{k^2}$.  
	Then with probability $1-O(k^{-24}\log^{-8} n{})$, the algorithm \textsc{GapClustering}{($V,h,\delta,b$)} successfully recover all the sub-clusters from the sampled set $T$, which correspond to true clusters  $V_1,\dots,V_h$. 
\end{lemma}
\begin{proof}%
	Let $T_i = V_i \cap T$, where $T$ is the sample set with $t$ vertices from the algorithm. Let $\lambda_1=\frac{bt}{4k^2}$. Note that $\lambda_1^2/t \geq 4\log n$ by our setting $t=\frac{8c_0 k^4\log n}{b^2\delta^2}$.  

	We first note that (over the randomness of sampling the vertex set $T$)
	\begin{itemize}
		\item \text{for any $i\leq h$, } it holds that $\E[|T_i|]\geq (\frac{1}{k}-\frac{h b}{k^2})t$. Thus, by Chernoff--Hoeffding bound (Theorem \ref{thm:chernoff}), with probability at least  $1-e^{-2 \lambda_1^2/t}\geq 1-n^{-8}$, 
		\begin{eqnarray}
			|T_i|\geq \left(\frac{1}{k}-\frac{hb}{k^2}\right) t - \lambda_1 = \left(\frac{1}{k}-\frac{(h+1/4)b}{k^2}\right)t \label{eqn:upper1}
		\end{eqnarray}
		
		\item \text{for any $i> h$, } it holds that $\E[|T_i|]< (\frac{1}{k}-\frac{(h+1)b}{k^2})t=(\frac{1}{k}-\frac{hb}{k^2})t-\frac{bt}{k^2}$. Thus, with probability at least $1-e^{-2 \lambda_1^2/t}\geq 1-n^{-8}$, 
		\begin{eqnarray}
			|T_i|< \left(\frac{1}{k}-\frac{hb}{k^2}\right)t-\frac{bt}{k^2} + \lambda_1 \leq  \left(\frac{1}{k}-\frac{(h+3/4)b}{k^2}\right)t \label{eqn:lower1}
		\end{eqnarray}
	\end{itemize}
	In the following, we assume the inequalities (\ref{eqn:upper1}) and (\ref{eqn:lower1}) hold for all $i\leq k$, which occur with probability at least $1-n^{-7}$ by the union bound. 
	
	Now we analyze the vertex degrees of vertices in the queried graph $H_T$. We first note that for any $v \in T_i$, its expected degree is  
	\[
	\left(\frac{1}{2}+\frac{\delta}{2}\right) |T_i| + \left(\frac{1}{2}-\frac{\delta}{2}\right) |T\setminus T_i| = \left(\frac{1}{2} -\frac{\delta}{2}\right) |T| + \delta |T_i|
	\]

	Let $\lambda_2=\frac{bt\delta}{4k^2}$. Note that $\lambda_1^2/t \geq 4\log n$ by our setting. Now we have that
	
	\begin{itemize}
		\item for any $i\leq h$ and vertex $v \in T_i$, then its expected degree is at least 
		\[
		\left(\frac{1}{2} -\frac{\delta}{2}\right) t + \delta b_h \geq \left(\frac{1}{2} -\frac{\delta}{2}\right) t + \delta\cdot  \left(\frac{1}{k}-\frac{(h+1/4)b}{k^2}\right)t.
		\]
		Thus, over the randomness of querying the oracle regarding vertices in $T$, with probability at least $1-e^{-2 \lambda_2^2/t}\geq 1-n^{-8}$, the degree of $v$ is at least
		\[
		\left(\frac{1}{2} -\frac{\delta}{2}\right) t +  \delta\cdot  \left(\frac{1}{k}-\frac{(h+1/4)b}{k^2}\right)t -\lambda_2 =\left(\frac{1}{2} -\frac{\delta}{2}\right) t +  \delta\cdot  \left(\frac{1}{k}-\frac{(h+1/2)b}{k^2}\right)t
		\]
		
		\item for any $i>h$ and  vertex $v \in T_i$, its expected degree is less than 
		\[
		\left(\frac{1}{2} -\frac{\delta}{2}\right) t + \delta |T_i|\leq \left(\frac{1}{2} -\frac{\delta}{2}\right) t + \delta\cdot  \left(\frac{1}{k}-\frac{(h+3/4)b}{k^2}\right)t
		\]
		
		Thus, with probability at least $1-e^{-2 \lambda_2 ^2/t}\geq 1-n^{-8}$, the degree of $v$ is less than
		\[
		\left(\frac{1}{2} -\frac{\delta}{2}\right) t + \delta\cdot  \left(\frac{1}{k}-\frac{(h+3/4)b}{k^2}\right)t+\lambda_2=\left(\frac{1}{2} -\frac{\delta}{2}\right) t +  \delta\cdot  \left(\frac{1}{k}-\frac{(h+1/2)b}{k^2}\right)t
		\]
	\end{itemize}
	
	Let $d_h:=(\frac{1}{2} -\frac{\delta}{2}) t +  \delta\cdot  \left(\frac{1}{k}-\frac{(h+1/2)b}{k^2}\right)t=\frac{t}{2}-\left(\frac12-\frac{1}{k}+\frac{(h+1/2)b}{k^2}\right)\delta t$.  That is, with probability at least $1-n^{-7}$, all vertices in $T_{1},\dots,T_h$ have degree at least $d_h$, and all vertices in $T_{h+1},\dots,T_k$ have degree less than $d_h$. Then by the description of the algorithm,  $T'=\cup_{i\leq h} T_i$. 
	
	Now we note that $H_{T'}\sim \rG(T_1,\dots,T_h,\delta)$, and that the number 
	of clusters in $H_{T'}$ is $h$. Now we apply \textsc{BalPartition}($H_{T'},h,\delta,b''$) on $H_{T'}$.  Recall that we have chosen $t=\frac{8c_0 k^4\log n}{b^2\delta^2}$. Note that we only need to consider the case that $t\leq n$ (as otherwise, we can simply query the whole graph). 
	Now we note that 
	\[t\geq |T'|\geq \sum_{i=1}^h|T_i|\geq h\cdot\left(\frac{1}{k}-\frac{(h+1/4)b}{k^2}\right)t\geq \frac{ht}{2k}
	\]
	Furthermore, we know for each $i\leq h$, 
	\[
	|T_i|\geq \left(\frac{1}{k}-\frac{(h+1/4)b}{k^2}\right)t\geq \frac{t}{2k}\geq \frac{|T'|}{2k} =\frac{h}{2k}\cdot\frac{|T'|}{h}.
	\]	
	
	Thus, if we set $b''=\frac{h}{2k}$, then the partition $T_1,\cdots,T_h$ is $b''$-balanced. Note that $|T'|\geq \frac{ht}{2k}\geq \frac{4 c_0 k^3 \log n}{b^2\delta^2}$. Thus, 
	\[
	\log|T'|\leq \log t \leq  \log n
	\]
	\[
	\frac{|T'|}{\log|T'|} \geq \frac{4c_0k^2}{h^2}\cdot \frac{h^2}{\delta^2}\cdot \frac{\log n}{\log t}\geq \frac{c_0}{b''^2}\cdot \frac{h^2}{\delta^2}
	\]
	Thus by Theorem \ref{thm:Vu18}, the algorithm \textsc{BalPartition}($H_{T'},h,\delta,b''$) successfully recover all the clusters $T_1,\dots,T_h$ with probability at least $1-|T'|^{-8}\geq 1-O((b\delta)^{16} k^{-24}\log^{-8} n{})$.

\end{proof}

\subsection{Finding a good index $h$}

In the previous section, we presented an algorithm for finding clusters assuming that the index $h$ that corresponds to the size-gap is known, and we have shown that the algorithm \textsc{GapClustering}($V,h,\delta,b$) outputs $h$ sub-clusters from the sampled set $T$.  However, in the general case, we do not know this index $h$. To handle this issue, %
we enumerate all possible candidates $h$ for $1\leq h\leq k$, and use a subroutine to test if the current candidate $h$ is `right' or not, which in turn makes use of a procedure for testing the bias of a given set. 
We first describe the algorithm for testing the bias of a set. Its performance is guaranteed in Lemma \ref{lemma:bias_two_cases}.%
\begin{algorithm}[H]
	
	\caption{\textsc{TestBias}($n, B,\eta$): test if a set $B$ is $(\eta,C)$-biased for some cluster $C$}
	\label{ }
	\begin{algorithmic}[1]
		\FOR{$i=1,\cdots, \frac{16k\cdot \log n}{b}$} 
		\STATE Randomly sample a vertex $v_i$ and query all the pairs $v_i,u$ for $u\in B$
		\IF{the number of `$+$' neigbhors of $v_i$ in $B$ is at least $(\frac{1}{2}+\frac{1}{2}\eta\delta)|B|$}
 {\RETURN $\Y$}
		\ENDIF
		\ENDFOR
		{ \RETURN $\N$}
		
	\end{algorithmic}
	
\end{algorithm}

\begin{lemma}\label{lemma:bias_two_cases}
	Let $B$ be a vertex set of size at least $\frac{64\log n}{\eta^2\delta^2}$. There exists one algorithm \textsc{TestBias}($n,B,\eta$) that with probability at least $1-n^{-7}$,  
	\begin{itemize}
		\item accepts $B$, if $B$ is $(\eta,C)$-biased for some cluster $C$ of size at least $\frac{bn}{k}$, i.e., $|B\cap C|\geq (1/2+\eta)\cdot |B|$
		\item rejects $B$, if $B$ is not $(\frac{\eta}{4},C)$-biased for any $C$, i.e., for any $C$, $|B\cap C|<(1/2+\frac{\eta}{4})\cdot |B|$.  
	\end{itemize}	
	
\end{lemma}

\begin{proof}%
	We first consider the case that $B$ is $(\eta, C)$-biased for some cluster $C$ of size at least $\frac{bn}{k}$. Note that with probability at least $1-n^{-8}$, one of the sampled $\frac{16k\log n}{b}$ vertices will belong to $C$, as $|C|\geq \frac{bn}{k}$.
	
	Furthermore, by the same calculations as the inequality (\ref{eqn:biaslowerbound}) in the proof of Lemma \ref{lemma:setB}, we know that with high probability, the $+$ neighbors of $v$ is at least 
	$(\frac{1}{2}+\frac{1}{2}\eta\delta)|B| $, 
	then \textsc{TestBias}($n,B,\eta$) will return $\Y$.
	
	Now suppose that $B$ is not $(\frac{\eta}{4},C)$-biased for any $C$. For any vertex $v\in V$, let $B_v$ be the set of vertices in $B$ in the same cluster as $v$. Then $|B_v|< (\frac12+\frac{\eta}{4})|B|$. The expected number of `$+$' neighbors of $v$ is 
	\[
	\left(\frac{1}{2}-\frac{\delta}{2}\right) |B| + \delta |B_v| \leq  \left(\frac{1}{2}-\frac{\delta}{2}\right)|B|+\delta \left(\frac12+\frac{\eta}{4}\right)|B| = \left(\frac12+\frac{\eta \delta }{4}\right)|B|
	\] 
	Let $\lambda=\frac{\eta\delta|B|}{4}$. Note that $\lambda^2/|B|\geq 4\log n$ as $|B|\geq \frac{64\log n}{\eta^2\delta^2}$. By Chernoff--Hoeffding bound, with probability at least $1-e^{-2t^2/|B|}\geq  1-n^{-8}$, the number of $+$ neighbors of $v$ is less than
	$(\frac12+\frac{\eta \delta }{4})|B|+\lambda=(\frac{1}{2}+ \frac{\eta \delta}{2})|B|$.
	In this case, the \textsc{TestBias}($n,B,\eta$) will return $\N$.
\end{proof}

Now we describe our idea for finding a good index $h$ and the corresponding sub-clusters. For each $h\in[k]$, we first ``pretend'' that the gap is $h$, and invoke \textsc{GapClustering}($V,h,\delta,b$) to find $h$ different sets $X_1,\cdots, X_h$ (or invoke \textsc{BalancedClustering}($V,h,\delta,b$) if $h=k$). Then we select sufficiently large subsets $X_i'\subset X_i$, $1\leq i\leq h$, and test if all of the sets $X_i'$ are sub-clusters by invoking a subroutine \textsc{TestBias}($n,X_i',\eta$). If so, we say the corresponding index $h$ is accepted, and the algorithm outputs the sets $X_1',\dots,X_h'$. 

\begin{algorithm}[H]
	\caption{\textsc{EnumerateIndex}($V,k,\delta,b,\eta$): find a good index $h$ and the corresponding sub-clusters }\label{}
	\begin{algorithmic}[1]
		\STATE Let $n=|V|$ 
		\FOR{$h=k,\dots,1$}
		\IF{$h==k$}
		\STATE Invoke \textsc{BalancedClustering}$(V,h,\delta,b)$ to find $h$ clusters $X_1,\dots, X_h$
		\ELSE 
		\STATE Invoke \textsc{GapClustering}$(V,h,\delta,b)$ to find $h$ clusters $X_1,\dots, X_h$
		\ENDIF
		\STATE\label{alg:xipri} For each $i\leq h$, let $X_j'$ be an arbitrary subset of $X_j$ of size $\frac{256\log n}{\eta^2\delta^2}$ 
		\IF{for all $i\leq h$, \textsc{TestBias}($n, X_i',\eta$) returns $\Y$}
 \RETURN $X_1',\dots,X_h'$
		\ENDIF
		\ENDFOR
	\RETURN \textbf{Fail}.
	\end{algorithmic}
	
\end{algorithm}
We have the following lemma regarding the performance guarantee of the above algorithm. %
\begin{lemma}\label{lemma:enumerate}
	Let $\eta^2/b\geq 64/c_0$, where $c_0$ is the constant from Theorem \ref{thm:Vu18}. It holds that with probability at least $1-n^{-6}$, 
	\begin{itemize}
		\item there exists an index $h\in[k]$ such that \textsc{EnumerateIndex}($V,k,\delta,b,\eta$) will output $h$ sets $X_1',\dots,X_h'$; 
		\item if $X_1',\dots,X_h'$ are the sets output by \textsc{EnumerateIndex}($V,k,\delta,b,\eta$), then each of them is $(\eta/4,C)$-biased for some cluster $C$. %
	\end{itemize}
\end{lemma}
\begin{proof}%
	If the instance is $b$-balanced, then we let $h=k$, and by Lemma \ref{lemma: McSherry}, \textsc{BalancedClustering}($V,h,\delta,b$) outputs all the sub-clusters $X_1,\dots,X_h$ from the sample set $T$. If the instance is not $b$-balanced, then by Lemma \ref{lemma:knowngap}, there exists an index $h\in[1,k-1]$ that corresponds to size-gap, and thus all the output sets $X_i$ by \textsc{GapClustering}$(V,h,\delta,b)$ are sub-clusters. In both cases, we know that $X_i$'s are $(\frac12,C)$-biased for some cluster $C$. Now by the previous argument, we can guarantee that each of the set $X_i$ has size at least $\frac{200c_0k\log n}{b\delta^2}$ (in case that $h=k$) or $\frac{4c_0k^3\log n}{b^2\delta^2}$ (in case that $h\leq k-1$), and thus larger than $\frac{256\log n}{\eta^2\delta^2}$, as $\eta^2/b\geq 64/c_0$ by assumption. Therefore, we can find subsets $X_i'$, $1\leq i\leq h$ of size $\frac{256\log n}{\eta^2}$ that are $(\frac12,C)$-biased for some cluster $C$. Thus, by Lemma \ref{lemma:bias_two_cases}, for all $i\leq h$, \textsc{TestBias}($n,X_i',\eta$) will be accepted with high probability. 

	Now we prove the second item of the lemma. Let $h$ be an index such that $1\leq h\leq k$. Let $X_1',\dots,X_h'$ be the sets corresponding to Step \ref{alg:xipri} of the algorithm \textsc{EnumerateIndex}. Let $\mathcal{E}_h$ denote the event that there exists one of the sets $X_i'$, $1\leq i\leq h$ is not $(\frac{\eta}{4},C)$-biased for any $C$. For any $h$ such that $\mathcal{E}_h$ holds, we know that with probability at least $1-n^{-7}$, one of tests \textsc{TestBias}($n,X_i',\eta$) will return $\N$ and thus 
	$h$ will not be accepted. Therefore, we can assume that for any $h\leq k$ such that $\mathcal{E}_h$ holds, $h$ will be rejected, which happens with probability at least $1-n^{-6}$. Furthermore, under this assumption, we have that if $h$ is accepted, then $\mathcal{E}_h$ does not hold, i.e., all the sets $X_i'$, $1\leq i\leq h$ are $(\frac{\eta}{4},C)$-biased for some cluster $C$. 	
\end{proof}

\subsection{The final algorithm}
Our algorithm is outlined as follows. %
\begin{itemize}
	\item Initialize $U=V$ and suppose the number of clusters in the current graph $G[U]$ is $k_c$, which equals $k$ at very beginning. Repeat the following until $U$ has small enough size or $k_c\leq 1$. 
	\begin{itemize}
		\item Use \textsc{EnumerateIndex}($U,k_c,\delta,b,\eta$) to find $h$ sets $X'_1,\dots,X'_h$, for some $h\leq k_c$. 
		\item Grow the found sets $X'_1,\dots,X'_h$ to find the clusters $C_1,\dots,C_h$ . 
		\item Update $k_c$ to be $k_c-h$, and remove all the clustered vertices from $U$. 
	\end{itemize}
	\item Output all the found clusters $C_i$'s. 
\end{itemize}
The psuedocode of the algorithm is as follows.

\begin{algorithm}[H]
	
	\caption{\textsc{NoisyClustering}($V,k,\delta$): the final clustering algorithm}\label{algorithm: noisyclustering}
	\begin{algorithmic}[1]
		\STATE Let $U=V$; let $k_c=k$ be the number of clusters in current graph; let $j=0$ be the number of clusters found so far; let $c_0$ be the universal constant from Theorem \ref{thm:Vu18}; let $b=\eta=0.1$
		\WHILE{$|U|\geq \frac{40000c_0k^4\log n}{\delta^2}$ and $k_c\geq 2$}%
		\STATE Invoke \textsc{EnumerateIndex}($U,k_c,\delta,b,\eta$) and let $X_1',\dots,X_h'$ denote the output $h$ sets. 
		
		\FOR{each $i\in[h]$}\STATE 
		$C_{j+i}\gets \{v\in U: \textsc{BelongToCluster}(v,X_i') \text{ returns $\Y$} \}$
		\STATE $U\gets U\setminus C_{j+i}$
		\ENDFOR
		\STATE $j\gets j+h$
		\STATE $k_c\gets k_c-h$
		\ENDWHILE			
		\RETURN all the clusters $C_i$'s
	\end{algorithmic}
	
\end{algorithm}

\begin{proof}[Proof of Theorem \ref{thm:main}] Since we have set $b=\eta=0.1$, it holds that $\eta^2/b\geq 64/c_0$ as $c_0\geq 1000$ by Theorem \ref{thm:Vu18}. By Lemma \ref{lemma:enumerate}, we know Algorithm \ref{algorithm: noisyclustering} will output $X_1',\dots,X_h'$ for some $h\leq k_c$, and each of these sets is $(\frac{\eta}{4},C)$-biased for some cluster $C$. Then by Lemma \ref{lemma:setB}, we can grow each $X_i'$ to get the true cluster $C$.  Note that at least one cluster will be found in each iteration, and the error probability in each iteration is at most $o_n(1)/k$ (by Lemma \ref{lemma: McSherry} and \ref{lemma:knowngap}). The final algorithm thus succeeds with probability $1-o_n(1)$ as there are at most $k$ iterations. The correctness of the algorithm then follows from the fact that the algorithm stops when all the $k$ clusters have been identified or the size of the remaining graph becomes smaller than $\frac{40000c_0k^4\log n}{\delta^2}$. 
	
	Now we bound the query complexity of the algorithm. Note that there are at most $k$ iterations. In each iteration, we invoke \textsc{EnumerateIndex} to try all $k$ possible values of $h$. For each $h$, we will sample at most $t=\frac{400c_0k^4\log n}{b^2\delta^2}=\frac{40000c_0k^4\log n}{\delta^2}$ vertices and query the induced subgraph by making $t^2$ queries for finding biased sets. To test the bias of each candidate set $X_i'$ (i.e., invoke \textsc{TestBias}($n,X_i',\eta$)), we only need to sample $\Theta(\frac{k\log n}{b})$ vertices and make $O(\frac{k\log n}{b}\cdot \frac{\log n}{\eta^2\delta^2})$ queries.  For the accepted index $h$, i.e., \textsc{EnumerateIndex} outputs $h$ sets $X_1,\dots,X_H$, we will make use of the subsets $X_1',\dots,X_h'$ to grow the clusters, and growing any set $X_i'$ to the true cluster requires at most $\frac{256\log n}{\eta^2 \delta^2} n$ queries. %
	Finally, we note that there can be at most $k$ subsets $X_i'$ throughout the whole procedure that we will use to grow the clusters. Thus, the total query complexity is $O(k^2t^2+kn\log n/\delta^2)=O(\frac{k^{10}\log^2 n}{\delta^4}+\frac{n k\log n}{\delta^2})$.%
	
	Regarding the running time, we let $T(t,k,\delta)=\textrm{poly}(t,k,1/\delta)$ denote the running time of \textsc{BalPartition} (in Theorem \ref{thm:Vu18}) on a set of size $t$. The running time for \textsc{TestBias}($n,T_i,\eta$) is proportional to the size $T_i$ and the running time of using \textsc{BelongToCluster} to identify each cluster is at most $tn$. Thus, the total running time is $O(k^2T(t,k,\delta)+kn\log n/\delta^2)=O((\frac{k\log n}{\delta})^C+\frac{nk\log n}{\delta^2})$, for some constant $C>0$.  
\end{proof}

\bibliographystyle{plainnat}
\bibliography{clustering}

\appendix
\begin{center}\huge\bf Appendix \end{center}

\section{Preliminaries}

We will make use of the following Chernoff--Hoeffding bound (see Theorem 1.1 in \cite{dubhashi2009concentration}).
\begin{theorem}[The Chernoff--Hoeffding bound]\label{thm:chernoff}
	Let $t\geq 1$. Let $X:=\sum_{1\leq i\leq t}X_i$, where $X_i, 1\leq i\leq t$, are independently distributed in $[0,1]$. Then for all $\lambda>0$, 
	\[
	\Pr[X>\E[X]+\lambda], \Pr[X<\E[X]-\lambda] \leq e^{-2\lambda^2/t}.
	\]
\end{theorem}
\section{Deferred Proofs from Section \ref{sec:subrountines}}\label{app:section_subrountines}
\subsection{Proof of Theorem \ref{thm:Vu18}}\label{app:Vu18}

We use $G\sim $ SBM($N,k,p,q$) to denote that the graph $G$ is generated from the SBM($N,k,p,q$) model. 
Let $C_u$ be the cluster that contains $u$, for any $u\in V$. The following was shown by  \citet{vu2018simple}.

\begin{theorem}[Theorem 1.2 in \cite{vu2018simple}]\label{thm:Vu_original}
	Let $G\sim$ SBM($N,k,p,q$). Let $s$ be the size of the minimum cluster. There exists a universal constant $c_1>20$ such that the following holds. Assume that 
	\[
	\sigma:=\sqrt{\max\{p(1-p),q(1-q)\}}\geq c_1\log N/N, s\geq c_1\log N, \text{ and } k=o((N/\log N)^{1/2}).
	\]
	Suppose further that for any $u,v$ that belong to two different clusters
	\[
	\sqrt{|C_u|+|C_v|}(p-q) \geq c_1 \left(\sigma \sqrt{\frac{N}{s}}+\sqrt{\log N}\right).
	\]
	Then there exists a polynomial time algorithm $\mathcal{A}$ that recovers all the clusters $V_1,\cdots,V_k$ of $G$, with probability at least $1-N^{-8}$.   
\end{theorem}
Now we show that Theorem \ref{thm:Vu18} can be derived the above theorem.

\begin{proof}[of Theorem \ref{thm:Vu18}]
	Note that to recover the clusters of $G\sim\rG(V_1,\dots,V_k,\delta)$, it suffices to consider the SBM($N,k,p,q$) model with 
	$N=n$, $k$ and  $p=\frac{1}{2}+\frac{\delta}{2}$ and $q=\frac{1}{2}-\frac{\delta}{2}$. Furthermore, since the corresponding partition is $b$-balanced,  the size of the smallest cluster is $s\geq \frac{bn}{k}$. Let $c_0=8c_1^2$, where $c_1$ is the universal constant from Theorem \ref{thm:Vu_original}.

	Now we claim that the precondition of Theorem \ref{thm:Vu_original} is satisfied. By the assumption that $n\geq c_0(k^2\log n)/(b^2\delta^2)$, it hols that $k =o((N/\log N)^{1/2})$ and $s\geq \frac{bN}{k}\geq c_1\log N$. Note further that
	\begin{align*}
		\sigma=\sqrt{(\frac12+\frac{\delta}{2})(\frac12-\frac{\delta}{2})}=\sqrt{\frac14-\frac{\delta^2}{4}} \in [\frac{\sqrt{3}}{4}, \frac{1}{2}] \qquad \Longrightarrow \qquad \sigma \geq c_1 \log N/N
	\end{align*}
	where we used the assumption that $\delta\leq\frac{1}{2}$ and that $n$ is sufficiently large. 
	
	Furthermore, for any two different clusters, we have $|C_u|+|C_v|\geq \frac{2bn}{k}$. Note that
	\begin{align*}
		p-q=\delta,  \qquad &\Longrightarrow\qquad \sqrt{|C_u|+|C_v|}(p-q) \geq \delta\sqrt{\frac{2bn}{k}}, \\
		\sqrt{\frac{N}{s}}\leq \sqrt{\frac{k}{b}}, \qquad \sqrt{\log N}=\sqrt{\log n} \qquad &\Longrightarrow\qquad \sigma \sqrt{\frac{N}{s}}+\sqrt{\log N}\leq \frac{1}{2}\sqrt{\frac{k}{b}}+\sqrt{\log n}
	\end{align*}
	
	Then by the precondition that
	\[
	n\geq c_0 \frac{k^2}{b^2\delta^2}\log n\geq c_1^2 \left(\frac{k^2}{4b^2\delta^2}+\frac{k\log n}{b\delta^2}\right)
	\]
	we have that 
	\[
	\delta^2\frac{2bn}{k} \geq 2c_1^2(\frac{k}{4b}+\log n) \Longrightarrow \delta\sqrt{\frac{2bn}{k}} \geq c_1 \left(\frac12\sqrt{\frac{k}{b}}+\sqrt{\log n}\right),
	\]
	where we used the inequality $2x^2+2y^2\geq (x+y)^2$. Thus, 
	\[
	\sqrt{|C_u|+|C_v|}(p-q) \geq c_1 \left(\sigma \sqrt{\frac{N}{s}}+\sqrt{\log N}\right).
	\]
	
	Therefore, by Theorem \ref{thm:Vu_original}, with probability at least $1-n^{-8}$, we can recover all the clusters $V_1,\dots,V_k$ in polynomial time. 
	
\end{proof}

\end{document}